\documentclass[letterpaper]{article} %
\usepackage{aaai25}  %
\usepackage{times}  %
\usepackage{helvet}  %
\usepackage{courier}  %
\usepackage[hyphens]{url}  %
\usepackage{graphicx} %
\usepackage{natbib}  %
\usepackage{caption} %
\usepackage[utf8]{inputenc}
\DeclareUnicodeCharacter{2265}{$\geq$}

\usepackage{algorithm}
\usepackage{algorithmic}
\usepackage{newfloat}
\usepackage{listings}
\usepackage{csquotes}
\DeclareCaptionStyle{ruled}{labelfont=normalfont,labelsep=colon,strut=off} %
\floatstyle{ruled}
\newfloat{listing}{tb}{lst}{}
\floatname{listing}{Listing}
\usepackage{mathtools}
\usepackage{amssymb, amsmath, amsthm}
\usepackage{thmtools, thm-restate}
\usepackage{cleveref}
\usepackage{enumitem}
\setlist{noitemsep}
\newtheorem{theorem}{Theorem}
\newtheorem*{theorem*}{Theorem}

\newtheorem{lemma}[theorem]{Lemma}
\newtheorem{corollary}[theorem]{Corollary}

\theoremstyle{definition}
\newtheorem{definition}[theorem]{Definition}
\theoremstyle{remark}
\newtheorem{example}[theorem]{Example}

\nocopyright %
\usepackage{pgf}
 \usepackage{lmodern}
\title{\LTLf{} Synthesis Under Unreliable Input (Extended Preprint)\footnote{This is the extended preprint of the conference paper with the same title presented at AAAI 2025.  It includes proof sketches of the main theorems in the text, with full proofs in \Cref{apx:proofs}. Additionally, the examples used are described in more detail in \Cref{apx:sheep,apx:hiker,apx:trap}.
}}
\author {
    Christian Hagemeier\textsuperscript{\rm 1},
    Giuseppe De Giacomo\textsuperscript{\rm 1},
    Mose Y. Vardi\textsuperscript{\rm 2}
}
\affiliations {
    \textsuperscript{\rm 1}University of Oxford, Oxford, UK\\
    \textsuperscript{\rm 2}Rice University, Houston, Texas, USA\\
    christian@hagemeier.ch, giuseppe.degiacomo@cs.ox.ac.uk, vardi@cs.rice.edu
}

\usepackage{xspace}
\newcommand{\LTL}{{\sc ltl}\xspace}

\newcommand{\LTLf}{{\sc ltl}$_f$\xspace}
\newcommand{\LDLf}{{\sc ldl}$_f$\xspace}

\newcommand{\QLTL}{{\sc qltl}\xspace}

\newcommand{\QLTLf}{{\sc qltl}$_f$\xspace}
\newcommand{\U}{\mathsf{U}}
\newcommand{\Always}{\raisebox{-0.27ex}{$\square$}}
\newcommand{\Next}{\raisebox{-0.27ex}{\LARGE$\circ$}}
\newcommand{\Wnext}{\raisebox{-0.27ex}{\LARGE$\bullet$}}
\newcommand{\lUntil}{\mathop{\U}}

\newcommand{\last}{\mathit{last}}

\newcommand{\Eventually}{\Diamond}
\newcommand*{\logeq}{\ratio\Leftrightarrow}
\usepackage{todonotes}

\usepackage{url}
\usepackage{comment}
\begin{document}

\maketitle

\begin{abstract}
We study the problem of realizing strategies for an \LTLf{} goal specification while ensuring that at least an \LTLf{} backup specification is satisfied in case of unreliability of certain input variables. We formally define the problem and characterize its worst-case complexity as 2EXPTIME-complete, like standard \LTLf{} synthesis.
Then we devise three different solution techniques:  one based on direct automata manipulation, which is 2EXPTIME, one disregarding unreliable input variables by adopting a belief construction, which is 3EXPTIME, and one leveraging second-order quantified \LTLf{} (\QLTLf), which is 2EXPTIME and allows for a direct encoding into monadic second-order logic, which in turn is worst-case nonelementary. 
We prove their correctness and evaluate them against each other empirically.  
Interestingly, theoretical worst-case bounds do not translate into observed performance; the MSO technique performs best, followed by belief construction and direct automata manipulation. As a byproduct of our study, we provide a general synthesis procedure for arbitrary \QLTLf{} specifications.
\end{abstract}
\begin{links}
\link{Code}{https://github.com/whitemech/ltlf-synth-unrel-input-aaai2025}
\end{links}

\section{Introduction}
One of the key challenges in Artificial Intelligence is to equip intelligent agents with the autonomous capability to deliberate and execute complex courses of action to accomplish desired tasks, see, e.g. the work on reasoning about action \cite{Reiter01} and on planning  \cite{GNT2016}. 

The problem we study is related to \emph{reactive synthesis} in Formal Methods \cite{PnRo89,Finkbeiner16,EhlersLTV17}, which shares deep similarities with planning in fully observable nondeterministic domains (FOND, strong plans \citep{CimattiRT98,geffner2013concise}).
A reactive synthesis setting is characterized by (boolean) variables, partitioned into input and output variables, changing over time. 
We have two entities acting in this setting: the environment that controls the input variables (or fluents, in planning terminology), and the agent who controls the output variables (or actions, in planning terminology). 
Given a specification, the agent has to find a strategy (or plan in planning terms) to choose its outputs to fulfil the specification by suitably reacting to the inputs chosen (possibly adversarially) by the environment.

In Formal Methods, the most common specification formalism is Linear Temporal Logic (\LTL) \cite{Pnueli77}. 
In AI, a finite trace variant of \LTL (\LTLf) is popular \cite{GabbayPSS80,baier2006planning,DegVa13,DR-IJCAI18}. The interest in finite traces is due to the observation that, typically, intelligent agents are not dedicated to a single task (specification) all their lives but are supposed to accomplish one task after another. In this paper, we will focus on \LTLf as a specification language.

Note that in reactive synthesis, at any point, the agent observes the current input and, based on its value (along with previous input values), decides how to react by choosing an appropriate output towards satisfying the specification.
Interestingly, machine-learning techniques are bringing about notable advancements in sensing technologies (i.e., generation of input), as showcased by the success of vision and autonomous driving techniques.
However, machine-learning techniques are typically data-oriented and hence have a statistical nature 
that may generate occasional unreliability of the input produced.
Consider a surgical robot that uses machine-learning models to interpret sensory data and guide its actions with precision during complex medical procedures.
Due to the inherent imprecision of the input models, it may misinterpret a critical piece of data regarding the patient’s anatomy (for instance, the robot might incorrectly identify a blood vessel as a different tissue type due to subtle variances in the imaging data that were not accounted for in the training phase of the AI model). This misinterpretation can lead the robot to make an inaccurate incision, potentially causing unintended harm to the patient.

In this paper, we study how to address such potential unreliability in the values of input variables in reactive synthesis.  
One way to address this unreliability is to disregard the unreliable input completely and not consider it in choosing the next output \cite{BloemC0S19}. This is related to planning/synthesis under partial observability \cite{Rintanen04a,de2016ltlf,Kupferman2000, 2015ehlers_estimator}. 
However, this might be too radical and could drastically reduce the agent's ability to operate, considering that the unreliability we are considering is only occasional. Our objective instead is to ensure that the system maintains functionality and adheres to critical specifications, despite uncertain and unreliable inputs. If the uncertainty is quantifiable, we could rely on probabilities turning to MDPs \cite{BaKG08,geffner2013concise} or Stochastic Games \cite{Kwiatkowska16}. Yet, as stated in a report by the White House, ``software does not necessarily conform neatly to probabilistic distributions, making it difficult to apply statistical models or predictions commonly used in other scientific disciplines''~\citep{WHR2024}. Here we aim at exploring a novel synthesis method to manage the potential unreliability of input variables obtained without relying on probabilities. 

The crux of our approach is not to give up on using input variables that might be unreliable but to complement them with the guarantee that even when they behave badly, some safeguard conditions are maintained.
Specifically, we consider two models simultaneously, a \emph{brave model} where all input variables are considered reliable (as usual in synthesis), and a \emph{cautious one}  where unreliable input is projected out and discarded \cite{de2016ltlf}.
Using these two models, we devise a strategy that simultaneously fulfils the task objectives completely if the input variables behave correctly and maintains certain fallback conditions even if the unreliable input variables behave wrongly. 

Our contributions are the following:
\begin{itemize}
    \item A formalization of \LTLf synthesis under unreliable input, which follows the framework described above.
    \item The computational characterization of the problem in terms of worst-case complexity as 2EXPTIME-complete, as standard \LTLf synthesis.
    \item Three provably correct synthesis techniques, one based on a direct method, one based on a belief construction and one based on solving synthesis for \QLTLf formulas, where \QLTLf, second-order quantified \LTLf, is the finite trace variant of \QLTL \cite{SistlaVW85,CalvaneseGV02}.
    \item The three techniques have different complexities: the direct one is 2EXPTIME, the one based on a belief construction is 3EXPTIME, and one based \QLTLf is 2EXPTIME,  but also admits a direct encoding in monadic second-order logic over finite sequences (MSO), which is non-elementary.
    \item An experimental assessment of the three synthesis techniques. Interestingly, the theoretical worst-case bounds do not translate into observed performance; the MSO technique performs best, followed by belief construction and direct automata manipulation.\footnote{Observe that, following the state-of-the-art, the implementation of all the techniques makes use of the tool MONA \cite{KlaEtAl:Mona}, which is based on first-order/monadic second-order logic over finite sequences, both non-elementary in the worst case.} 
\end{itemize}
As a side result, we present a synthesis technique for arbitrary \QLTLf specifications. 

\section{Preliminaries}\label{sect:prelim}
\begin{definition}\label{def:langltlf}
	The language of \LTLf is defined by  
	\[ \varphi \coloneqq A \mid \lnot \varphi \mid \varphi \wedge \psi \mid \Next \varphi \mid \phi \lUntil  \psi \]
	with $A \in \mathcal P$, where $\mathcal P$ is a set of propositional variables.
\end{definition} 
Other modalities, like Weak Next ($\Wnext \varphi \coloneqq \lnot \Next \lnot \varphi$), Eventually ($\Diamond$), and Always ($\Always$), can be defined.
\begin{definition}[\LTLf{} semantics]
	Given a trace $t~\in~(2^{\mathcal P})^{+}$, the satisfaction relation $t, i \vDash \varphi$ is defined inductively for $1 \leq i \leq \mathsf{length}(t) = \last{}$: 
	\begin{itemize}\itemsep=0pt
		\item $t, i \vDash A $ iff $A \in t(i)$ 
		\item $t, i \vDash \lnot \varphi$ iff $t, i \nvDash \varphi$
		\item $t, i \vDash \varphi  \wedge \psi$ iff $t , i \vDash \varphi$ and $t , i \vDash \psi$
		\item $t, i \vDash \Next \varphi$ iff $i < \last{}$ and $t , i+1 \vDash \varphi$
		\item $t, i \vDash \varphi \, \lUntil \, \psi$ iff for some $j$ with $i \leq j \leq \last{}$ we have $t, j \vDash \psi$ and for all $k$ with $i \leq k < j$ we have $t, k \vDash \varphi$.
	\end{itemize}
\end{definition}
We say that a trace satisfies an \LTLf{}-formula, written as $t \vDash \varphi$, iff $t, 1 \vDash \varphi$.

\textbf{Synthesis under full observability.}
Classical \LTLf{} \textbf{synthesis} refers to a game between an agent and the environment. 
Both control a subset of the variables of an \LTLf{}-formula $\varphi$, which the agent tries to satisfy. 

 An agent's strategy is a function \(\sigma: (2^\mathcal{X})^* \rightarrow 2^\mathcal{Y}\). The strategy \(\sigma\) realizes a formula \(\varphi\) if, for any infinite sequence \(\lambda = (X_0, X_1, \dots) \in (2^{\mathcal{X}})^{\omega}\), there exists a \(k \in \mathbb{N} \) such that the finite trace \( t = ((X_0, Y_0), \dots, (X_k, Y_k))\) satisfies \(\varphi\) (i.e. $t \vDash \varphi$), where \(Y_i = \sigma(X_0, \dots, X_{i-1})\).

\textbf{DFA Games.}
The standard technique for solving \LTLf{} synthesis works by reducing it to solving a reachability game on a DFA~\cite{de2015synthesis}. 
A DFA game is also played between two players, the agent and the environment. They have corresponding sets of boolean variables $\mathcal{X}, \mathcal{Y}$. 
The specification of a game is given by a DFA $\mathcal{G} = (2^{\mathcal X \cup \mathcal Y }, S, s_0, \delta, F)$ where $2^{\mathcal X \cup \mathcal Y}$ is the alphabet, $S$ the set of states, the initial state $s_0$, total transition function $\delta: S \times 2^{\mathcal X \cup \mathcal{Y}} \rightarrow S$ and the set of final states $F \subseteq S$.  A play on a DFA game is a sequence $\rho = ((s_{i, 0}, X_0 \cup Y_0), (s_{i, 1}, X_1 \cup Y_1), \dots ) \in (S \times 2^{\mathcal X \cup \mathcal Y})^{+}$ with $s_{i, j+1} = \delta(s_{i,j}, X_i \cup Y_i)$. Such a play is winning if it ends in a final state. 
We say that a player has a winning strategy in a DFA game if they can choose the variables in a way that guarantees to end up in a final state regardless of how the other player responds.

\section{\LTLf{} Synthesis under Unreliable Input}\label{sect:problem}

Before formalizing the problem, we introduce some additional notation. The projection function \(\mathsf{proj}_{\mathcal{V}}(t)\) removes all variables in \(\mathcal{V}\) from a trace \(t\) over propositional variables \(\mathcal{P}\). The expansion function \(\mathsf{exp}_{\mathcal{V}}(\widehat{t})\) takes a trace \(\widehat{t}\) over variables \(\mathcal{\hat{P}}\) (with \(\mathcal{V} \cap \mathcal{\hat{P}} = \varnothing\)) and returns all traces \(t\) by setting variables in \(\mathcal{V}\) in every possible way at all instants of \(\widehat{t}\). For traces \(t\) and \(t'\), \(t \sim_{-\mathcal{V}} t'\) means \(t \in \mathsf{exp}_{\mathcal{V}}(\mathsf{proj}_{\mathcal{V}}(t'))\).
Slightly abusing notation, we use the same syntax to denote two subsets of $\mathcal P$ being equal up to elements of $\mathcal V$.
We model uncertainty about the environment by assuming the agent cannot rely on the readings of certain environment variables; thus, we require that after any changes in these readings still satisfy a backup condition; this leads to the following formalization:

\begin{definition}[\LTLf{} synthesis under unreliable input]
Given \LTLf{}-formulas $\varphi_m, \varphi_b$ over variables $\mathcal X \uplus \mathcal Y$, called respectively the main and backup specification, and a partitioning $\mathcal X = \mathcal X_{rel} \uplus \mathcal X_{unr}$ of the input variables into reliable and unreliable ones respectively, 
solving 
\emph{\LTLf{} synthesis under unreliable input} amounts to finding
a strategy $\sigma:(2^\mathcal{X})^* \to 2^\mathcal{Y}$ such that for any infinite sequence of variables $\lambda = (X_0, X_1, \ldots) \in  {(2^{\mathcal X})}^{\omega}$ there is an index $k \in \mathbb N$ such that
\begin{enumerate}
\item The finite trace $t = ((X_0, Y_0), \dots, (X_k, Y_k))$ with $Y_i = \sigma(X_0, \dots, X_{i-1})$ satisfies $\varphi_m$, i.e.,  $t \vDash \varphi_m$, 
 \item and every $t'$ with $t' \sim_{- \mathcal X_{unr}} t$ satisfies $\varphi_b$, i.e., $t' \vDash \varphi_b$.
\end{enumerate}
\end{definition}
Our problem extends standard \LTLf{} synthesis by using \(\top\) as a backup formula and \(\mathcal{X}_{unr} = \varnothing\). Since \LTLf{} synthesis is 2EXPTIME-complete \cite{de2015synthesis}, our problem is 2EXPTIME-hard. 
We later show a matching upper bound (see \Cref{thm:directcomplexity}). 
The problem is also related, however distinct, to \LTLf synthesis under partial observability~\citep{de2016ltlf}. 
If the main specification goal is trivial (i.e. setting $\varphi_m \coloneqq \top$), our problem degenerates into \LTLf synthesis under partial observability.

Let us illustrate the problem with some examples, which are designed such that by suitably choosing parameters, realizability can be controlled. The full descriptions and formulas can be found in \Cref{apx:sheep,apx:trap,apx:hiker}, respectively.

\begin{example}[Sheep]
A farmer wants to move her $n$ sheep from the left to the right of a river, always moving two sheep simultaneously (it is inspired by well-known puzzles~\citep{ferry}). However, some pairs of sheep may not like each other. 
If she can correctly observe the animosities, she wants to move all sheep; however, when her intuition about which sheep are compatible is unreliable, only a subset (her favourite sheep) must be moved.

Here, the agent has $n$ output variables, of which she can always activate two to request the environment to move the corresponding sheep. Additionally, if for any potentially forbidden pair of sheep $i,j$, an unreliable input variable $\mathsf{disallow_{i,j}}$ is activated, it indicates that $i$ and $j$ cannot be moved together.  This leads to synthesis for the formula $\varphi_{ag} \wedge (\varphi_{env} \supset \varphi_{goal})$, with different $\varphi_{goal}$ for main and backup. Here, the unreliable inputs are variables $\mathsf{disallow_{i,j}}$ for each pair of sheep that may not like each other. 
\end{example}

\begin{example}[Trap]
A robot is searching for a path from a vertex $v$ of a graph with $n$ vertices to a set of secure vertices. For simplicity of modelling, there are at most 2 outgoing edges from any vertex. However, the environment may control the state of a set of traps, allowing it to divert one of two edges to a different endpoint.
If the robot can correctly sense the state of the traps, it should move to a secure vertex; however, if it cannot correctly observe the traps, the same strategy should guarantee that the (larger) backup safety region is reached. 

We can model this with $t$ unreliable input variables $t_i$ that indicate the state of each trap. Additionally, the environment controls $\lceil \log_2(n) \rceil$ variables for denoting the current position. The main and backup specifications then both have the form $\varphi_{env} \rightarrow \bigvee_{v \in R} \Eventually (pos(v))$, where \( \bigvee_{v \in R} \Eventually (pos(v)) \) ensures that the agent ends up in the corresponding region.
\end{example}

\begin{example}[Hiker]
We model a hiker on a trail of length $n$ in the Alps who wants to eat berries, avoiding poisonous ones that could make her sick. 
Due to the similarity between poisonous and non-poisonous berries, she might consume poisonous ones if her senses are unreliable. 
Fortunately, natural medical herbs along the trail can cure sickness. Her main goal is to eat all non-poisonous berries, and even when her senses are unreliable, she wants to ensure that she is not sick by the end of the trail.

We can model this as an instance of synthesis under unreliable input by having $\mathcal X = \{ berry, poison, sick, herb, eot \}$ and $\mathcal Y = \{ eat, takeMedicine, collectMedicine\}$. We then have an environment specification $\varphi_e$ and an agent main and backup goal. The main specification then is simply \( \varphi_e \supset \Box (\mathsf{berry} \wedge \lnot \mathsf{poison} \supset \mathsf{eat} ) \wedge \Diamond (\mathsf{eot}). \) And the backup specification is $\varphi_e \supset \Diamond (\mathsf{eot} \wedge \lnot \mathsf{sick})$. The only unreliable input variable here is the variable that indicates whether the berry before the hiker (if existent) is poisonous, i.e. $\mathcal X_{unr} = \{\mathsf{poison} \}$.
The environment constraints specify how each of the variables changes; for example, for sickness, we have the following environment constraint (in the style of successor state axioms~\cite{reiter1991frame}):
\( \begin{array}{l}
     \Next sick  \equiv \Next \top \wedge ( \Next eat \wedge  berry \wedge  poison  )  \\
     \qquad{} \vee (sick \wedge \lnot (inbag \wedge takeMedication))
\end{array} \)

In general, the hiker does not have a strategy to solve this under unreliable input, as when there are no herbs along the path, she cannot guarantee fulfilling the backup specification under unreliable input. However, if we force there to be herbal medicine at least at one point of her trail, then the following strategy realises both goals: Eat all berries, collect the medicine when it is on the trail and finally, use it after eating all berries. The backup specification in this setting does not influence realisability but changes which strategies successfully realise the goal, as this blocks simple strategies such as eating all berries that appear non-poisonous and disregarding the medicine.
\end{example}

\section{Technique 1: Direct Automata}\label{sect:proj}

Recall that we can essentially view the problem as the agent having to satisfy two goals in the brave and cautious arena simultaneously. 
This suggests combining an arena for the main goal under full observability with an arena for the backup goal under partial observability. 
We first describe the construction before showing correctness. It has three main ingredients: an automaton recognizing the main formula, one for the backup formula under partial observability and the correct combination of these, the synchronous product.

\begin{definition}[Synchronous Product of DFAs]
Given two DFAs \(\mathcal{G}_i = (2^{\mathcal{X}} \times 2^{\mathcal{Y}}, S_i, s_{0,i}, \delta_i, F_i)\) (for \(i \in \{1, 2\}\), and defined over the same alphabet), we define their synchronous product as \(G_1 \otimes G_2 = (2^{\mathcal{X}} \times 2^{\mathcal{Y}}, S_1 \times S_2, (s_{0,1}, s_{0,2}), \delta', F')\), where \(\delta'((s_1, s_2), \sigma) = (\delta_1(s_1, \sigma), \delta_2(s_2, \sigma))\) and \(F' = F_1 \times F_2\).
\end{definition}
Let us now describe the construction of the automaton for the backup formula under partial observability:
\begin{itemize}
    \item First, we create an NFA for the complement of the formula, i.e. $A_{\lnot \varphi}$.
    \item Next, we existentially abstract over the unreliable inputs $U_1, \dots, U_n$, yielding an NFA $(A_{\lnot \varphi})_{\exists U_1, \dots, U_n}$, which we formally define in \Cref{def:existentialNFA}.
    \item Lastly, we determinize the NFA using subset construction. We then complement the final DFA obtaining our final automaton $\overline{\mathcal D((\mathcal A_{\lnot \varphi})_{\exists U_1, \dots, U_n})}$.
\end{itemize}

Formally, we can define the existential abstraction:
 \begin{definition}\label{def:existentialNFA}
 	Given an NFA $N = (2^{\mathcal X} \times 2^{\mathcal Y}, S, \delta, s_0, F)$, we define the existentially abstracted NFA $N_{\exists U_1, \dots, U_n} =  \left (2^{\mathcal X} \times 2^{\mathcal Y}, S, \delta', s_0, F \right )$ by setting 
\small
\begin{equation*}
\delta'(s, (X, Y)) = \{ s' \mid \exists X'. X \sim_{- U_1, \dots, U_n} X' \wedge s' \in \delta(s, (X', Y)) \}.
\end{equation*}
\normalsize

 \end{definition}
 
\begin{restatable}{theorem}{rsThmCorrectProject}\label{thm:projectcorrect}
    Solving synthesis for the synchronous product of the \LTLf{}-automaton for $\varphi_m$, $A_{\varphi_m}$, and the automaton $\overline{\mathcal D((\mathcal A_{\lnot \varphi})_{\exists U_1, \dots, U_n})}$ solves synthesis under unreliable input.
\end{restatable}
\begin{proof}
We prove (see \Cref{apx:proofs}) correctness by showing: (1) partitioning the problem into two arenas and combining them via synchronous product is valid; (2) the main arena correctness follows from the correctness of \LTLf{} to automata translation; (3) the projection automaton satisfies the backup condition. 
\end{proof}

\begin{theorem}[]\label{thm:directcomplexity}
    The outlined technique has worst-case complexity of 2EXPTIME.
\end{theorem}
\begin{proof}
    Constructing the DFAs for $\varphi_m$ is 2EXPTIME; the NFA for $\varphi_b$ is EXPTIME. Determinizing the NFA after projection adds another exponential. Synchronous product and reachability game are polynomial.
    \end{proof}
Combined with the 2EXPTIME-hardness of \LTLf{} synthesis, this establishes 2EXPTIME-completeness for synthesis under unreliable input:
\begin{theorem}\label{thm:problemcomplexity}
\LTLf{} synthesis under unreliable input is 2EXPTIME-complete.
\end{theorem}

This algorithm can be efficiently implemented in a symbolic synthesis framework. Here, given symbolic automata, the synchronous product simply corresponds to merging the sets of bits representing states and conjunction with the BDD representation for the final states (i.e.\  \citet{de2023a}). 
The construction of the NFA can be efficiently implemented symbolically by translating the negated formula into Pure-Past-LTLf, and thus constructing a DFA for the reverse language of $\lnot \varphi_b$ (as described in \citet{ZhuShufang2019FvSE}). Then the subset construction (for determinization), reversal and existential abstraction can be carried out efficiently on the BDDs. For instance, the subset construction step can be carried out symbolically by introducing one variable for each state of the NFA.
\section{Technique 2: Belief-States}\label{sect:belstate}
\citet{de2016ltlf} also investigate a second technique for generating an automaton to solve the game under partial observability, namely the belief-state construction, that we can basically use as an alternative to constructing an automaton for the backup formula under partial observability, keeping the other steps identical.
\begin{definition}[Belief State DFA Game]
Given a DFA game \(\mathcal{A} = (2^{\mathcal{X} \cup \mathcal{Y}}, S, s_0, F)\) with input variables partitioned into \(\mathcal{X} = \mathcal{X}_{rel} \uplus \mathcal{X}_{unr}\), we define the belief-state game \(\mathcal{G}^{rel}_{\mathcal{A}} = (2^{\mathcal{X} \cup \mathcal{Y}}, \mathcal{B}, B_0, \partial, \mathcal{F})\) as follows: \(\mathcal{B} = 2^{S}\) (the power set of states), \(B_0 = \{s_0\}\) (the initial state lifted to the power set), \(\partial : \mathcal{B} \times 2^{\mathcal{X} \cup \mathcal{Y}} \rightarrow \mathcal{B}\) defined by \scriptsize %
\[
    \partial(B, (X \cup Y)) = \{ s' \mid \exists s \in B \exists X'. X \sim_{-\mathcal{X}_{unr}} X' \wedge \delta(s, (X' \cup Y)) = s' \},
\] \normalsize
and \(\mathcal{F} = 2^F\) (the final states of the game).
\end{definition}
With this definition, we can show correctness and characterize the complexity.
\begin{restatable}{theorem}{thmCorrectBelief}
    Solving synthesis for the synchronous product of the \LTLf{}-automaton for $\varphi_m$, $A_{\varphi_m}$,  and the belief-state automaton $G^{rel}_{\mathcal A_{\varphi_b}}$ solves synthesis under unreliable input.
    \end{restatable}
\begin{theorem}
    Solving synthesis with backup using the belief-state construction yields a 3EXPTIME algorithm. 
\end{theorem}
\begin{proof}
    This follows from the fact that constructing the belief-state automaton takes 3EXPTIME (construction of the DFA from the formula takes 2EXPTIME, then the belief-state construction costs another exponential); generating the DFA for the main formula takes 2EXPTIME, the other steps are polynomial. 
\end{proof}

The belief state construction can also be implemented symbolically efficiently, similar to the subset construction for the direct approach~\citep{tabajara2020ltlf}.

\section{Quantified \LTLf{} Synthesis}\label{sect:qltlf}
Our third technique builds on translating the problem into synhesis for \QLTLf, which is a variant of \QLTL, that similarly adds second-order quantification over propositional variables to \LTLf.
In this section, we present a general algorithm for  \QLTLf synthesis.
We then cast synthesis under unreliable input as a special case of this problem and show that for such formulas, the general algorithm matches the 2EXPTIME bound.
A \QLTLf{} formula is given by the following grammar ($X$ denotes a second-order variable):\footnote{In \QLTLf and MSO, with a little abuse of notation, we do not distinguish between second-order variables and propositions - open variables act as propositions.}
\[ \varphi \coloneqq X \mid \lnot \varphi \mid \varphi \wedge \psi \mid \Next \varphi \mid \varphi \, \mathcal{U} \, \psi \mid \exists X. \varphi  \]
The satisfaction relation for \QLTLf{} is defined as for \LTLf{}, with an additional clause for the second-order quantifier:
\[ t, i \vDash \exists X. \varphi :\logeq \exists t'. t \sim_{-X} t' \wedge t', i \vDash \varphi \]
Universal quantification in \QLTLf{} is defined as $\forall X. \varphi \coloneqq \lnot \exists X. \lnot \varphi$; other modalities are defined as in \LTLf{}.
More informally, the existential quantifier '$\exists x. \varphi$' in \QLTLf{} states that there is at least a way to modify where in the trace a variable x holds, thereby making the formula $\varphi$ true. 
A formula is in prenex normal form (PNF) if it is of the form $Q_1 X_1. \, Q_2 X_2.\, \dots Q_n X_n. \, \varphi$ where $\varphi$ is an \LTLf{}-formula and $Q_i \in \{ \forall, \exists \}$. Any formula can be polytime-transformed into PNF, similar to \QLTL{} \citep{piribauer2021quantified}.  The alternation count of a formula in PNF is the number of indices $i$ s.t.\  $Q_i \neq Q_{i+1}$.%

\begin{restatable}{theorem}{thmQLTLfkplustwoExpTime}\label{thm:complexitySynthesisQltlf}
       Synthesis for a \QLTLf{}-formula $\psi$ with $k$ alternations can be solved in $(k+2)$-EXPTIME.  
\end{restatable}
\begin{proof}
Induction on alternation count; see \Cref{apx:proofs}. %
Base cases can be handled similarly to the projection construction and thus need at most 2-EXPTIME. The inductive step, essentially repeats this argument.
\end{proof}
\begin{restatable}{theorem}{thmReductionToQltlf}\label{thm:redToQltlf}
	A strategy $\sigma$ realizes the instance of \emph{\LTLf{} synthesis under unreliable input} with $\mathcal{X}_{unr} = \{ U_1, \dots, U_n \}$  iff it realizes synthesis for the \QLTLf{} formula $\varphi_m \wedge \forall U_1. \dots \forall U_n. \varphi_b$.
\end{restatable}
\begin{proof}
Follows by \QLTLf{} semantics, see \Cref{apx:proofs}. %
\end{proof}

From counting the alternations in the formula, we can deduce that this technique has optimal worst-case complexity.
\begin{theorem}
    Solving \LTLf{} synthesis under unreliable input by translating into \QLTLf has complexity 2EXPTIME.
\end{theorem}
\begin{proof}
    Follows from the formula in Theorem~\ref{thm:redToQltlf} having 0 alternations and \Cref{thm:complexitySynthesisQltlf}. 
\end{proof}

\section{Technique 3: MSO Encoding}\label{sect:mso}

Exploiting Theorem~\ref{thm:redToQltlf} we now propose a third solution technique for \LTLf{} synthesis under unreliable input. 
Specifically, we start from the \QLTLf specification $\varphi_m \wedge \forall U_1. \dots \forall U_n. \varphi_b$ %
translate it into monadic second-order logic (MSO), and then use MONA to obtain the DFA corresponding to the original specification (for synthesis under unreliable input).
Then we can solve the DFA game, just like in standard \LTLf{} synthesis. In fact, this approach also works for synthesis of arbitrary \QLTLf formulas. %

 	 	Formulas of MSO are given by the following grammar ($x,y$ denote first-order variables, and $X$ denotes a second-order variable): 
 	\[ \varphi \coloneqq X(x) \mid x < y  %
  \mid (\varphi \wedge \psi) \mid \lnot \varphi \mid \exists x. \varphi \mid \exists X. \varphi. \]

We then consider \emph{monadic structures} as interpretations that correspond to traces. We use the notation $t, [x/i] \vDash \varphi$ to denote that this interpretation of second-order variables (details in \Cref{apx:proofs}), with assigning $i$ to the FO-variable $x$ satisfies $\varphi$. 
We can then give a translation and show its correctness:
\begin{definition}
We define a translation $\mathsf{mso}$ from \QLTLf{} to MSO by setting (we use standard abbreviations for $succ(x,y)$, $x\leq y$, and $x\leq\last$):
\small
\[
\begin{array}{rcl}
\mathsf{mso}(X, x) &\coloneqq& X(x) \\
\mathsf{mso}(\lnot \varphi, x) &\coloneqq& \lnot \mathsf{mso}(\varphi, x) \\
\mathsf{mso}(\varphi \wedge \psi, x) &\coloneqq& \mathsf{mso}(\varphi, x) \wedge \mathsf{mso}(\psi, x) \\
\mathsf{mso}(\Next \varphi, x) &\coloneqq& \exists y. succ(x, y) \wedge \mathsf{mso}(\varphi, y) \\
\mathsf{mso}(\varphi \lUntil \psi, x) &\coloneqq& \exists y. (x \leq y \leq \last) \wedge \mathsf{mso}(\psi, y) \\
& & \quad \wedge \forall z. (x \leq z < y \rightarrow \mathsf{mso}(\varphi, z)) \\
\mathsf{mso}(\exists X. \varphi, x) &\coloneqq& \exists X. \mathsf{mso}(\varphi, x)
\end{array}
\] 
\normalsize
\end{definition}

 \begin{restatable}{theorem}{thmTranlsationmsocor}
For any closed \QLTLf{} formula \(\varphi\) and finite trace \(t\), \(t, i \vDash \varphi\) iff 
\( t, [x/i] \vDash \mathsf{mso}(\varphi, x) \).

 \end{restatable}
 \begin{proof}
     By an induction on $\varphi$; see \Cref{apx:proofs}.
 \end{proof}

This translation gives us the following technique to solve synthesis for a \QLTLf{} formula.
Once we have translated the formula to MSO, we can use MONA to obtain the DFA and then play the DFA game to solve synthesis.
\begin{theorem}\label{thm:solveQltlfUsing}
    The technique for \QLTLf synthesis is correct.
\end{theorem}
\begin{proof}
Follows from the correctness of the translation.
\end{proof}
As a result, we can correctly solve our problem:
\begin{theorem}
    The technique to solve synthesis under unreliable input with \LTLf main specification $\varphi_m$ and \LTLf backup specification $\varphi_b$, based on synthesis for the \QLTLf formula $\psi=\varphi_m \wedge \forall U_1. \dots \forall U_n. \varphi_b$, is correct.
\end{theorem}
\begin{proof}
Immediate by \Cref{thm:redToQltlf} and \Cref{thm:solveQltlfUsing}.
\end{proof}

As we already discussed, we can only upper-bound the runtime using MONA as worst-case non-elementary because of the MSO-to-DFA translation.
This technique can be easily implemented using Syft~\citep{2017zhu_syft}, an open-source \LTLf\ syntheziser.\footnote{The original Syft source code is available at \url{github.com/Shufang-Zhu/Syft}. For our modifications, see the earlier linked repository.}
The only modifications to Syft's pipeline happen before running synthesis; only the input we generate for MONA changes. 
It is interesting to notice that the implementations of the other techniques end up using MONA to create the DFA, too, since they use Syft.
\section{Empirical Evaluation}\label{sect:emp}
We implemented the algorithms to empirically evaluate their performance. While the direct technique is optimal in the worst case, our empirical results show that other techniques perform comparably or even better.

\textbf{Implementation.}
We built the techniques on top of the \LTLf{}-synthesis tool Syft, reusing existing implementations for belief-state and direct automata construction~\citep{tabajara2020ltlf} with some adjustments. Detailed instructions for compiling and using our implementation are in \Cref{apx:reprod}.

\textbf{Benchmarks.}
We used instances of the examples of varying sizes: hiker (trail length 10 to 50, increments of 5), sheep (even input sizes up to $n = 10$), and trap graphs (multiple sizes).
For sheep, we restricted our attention to even input sizes as otherwise, trivially synthesis under unreliable input is impossible since standard synthesis already is. 
For $n > 10$, DFA-generation using  MONA timed out. 

\textbf{Experimental Setup.} Experiments were conducted on a high-performance compute cluster running Red Hat Enterprise Linux 8.10. Tests used an Intel Xeon Platinum 8268 processor (2.9 GHz) with 256 GB RAM per test, executed on a single CPU core. 

\textbf{Experimental Results.} We report the results for each example; all terminating runs produced correct results.
Graphs of the runtime on different instances of the examples can be found in \Cref{fig:runtimeSheep,fig:runtimeTrap}. We used an average of 2 runs to produce the results. Given the deterministic nature of our algorithms, the additional run is only for double-checking.
The y-axis shows the total runtime in seconds (DFA construction and synthesis combined) for our test suite, while the x-axis lists each test name. Each test has three bars representing runtimes for the direct automata (blue), belief state (orange), and MSO approach (green). Striped bars reaching the red timeout line indicate tests that ran out of memory or exceeded 30 minutes.

The main findings are this: 
 Considering only the amount of test cases solved within the time limit, MSO performs best, with belief-state second and direct approach solving the least and noticeably smaller amount. We conjecture that this is partly due to MONA's efficient implementation and its already minimized DFAs, unlike the non-minimized product DFAs we compute, which worsen runtime for synthesis.

 In general, the runtime of MSO is lower (and, for bigger examples, orders of magnitude lower) than that of the other approaches.
For the hiker example, both MSO and belief state construction can solve all test cases up to $50$ in length. 
 In contrast, projection only manages to solve very small examples. 
Here, it is notable that many tests timed out during DFA construction, suggesting that constructing the automaton for the reversed language for these examples is hard.
 This pattern also shows up in the other examples, albeit less pronounced.      As discussed before, in sheep and trap, as the number of sheep or traps increases, the number of variables also increases, while for hiker, this is constant (6 input variables, 3 output variables); hence, hiker examples are solvable for larger values.
We performed a Wilcoxon-signed rank test with a significance level of $\alpha = .05$ , both for each example individually and for the whole dataset. The runtimes are significantly different between approaches, except for the runtime difference between direct and belief approaches on the trap examples. 
We compared our MSO synthesis technique's performance with synthesis for the main and backup formula under full/partial observability, respectively. Our runtime is proportional to the sum of both, with a maximum 2.5x overhead (c.f. \Cref{apx:figures}).
\begin{figure}
    \includegraphics{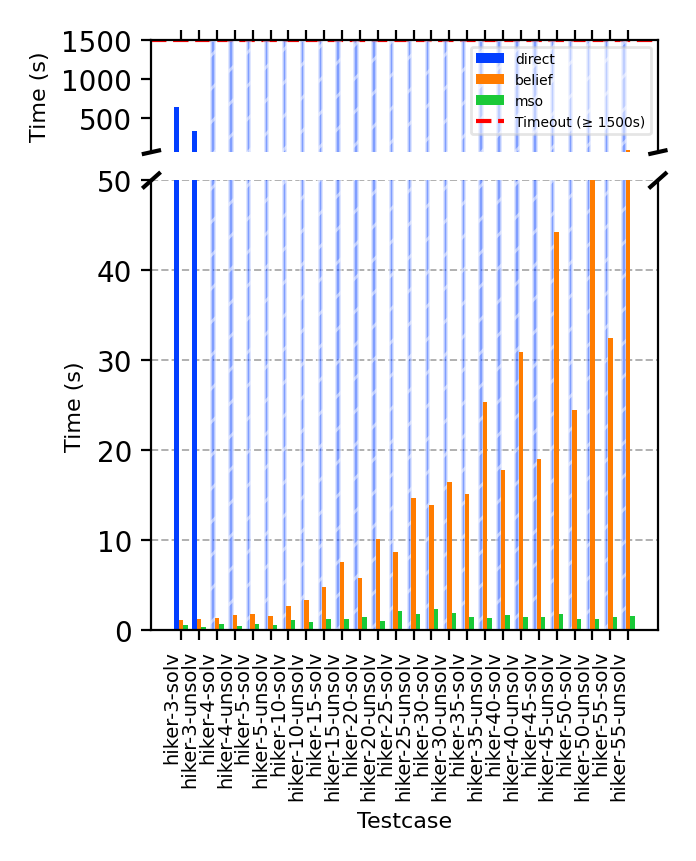}
    \label{fig:runtimeHiker}
    \includegraphics{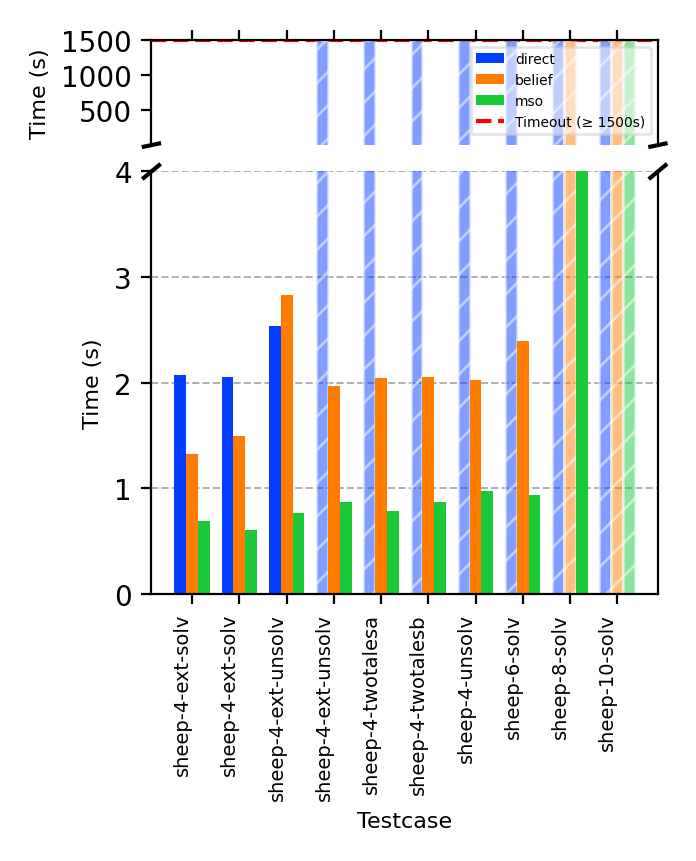}
    \caption{Runtime for the sheep and hiker examples.}
    \label{fig:runtimeSheep}
\end{figure}
\begin{figure}
        \includegraphics[scale=0.8]{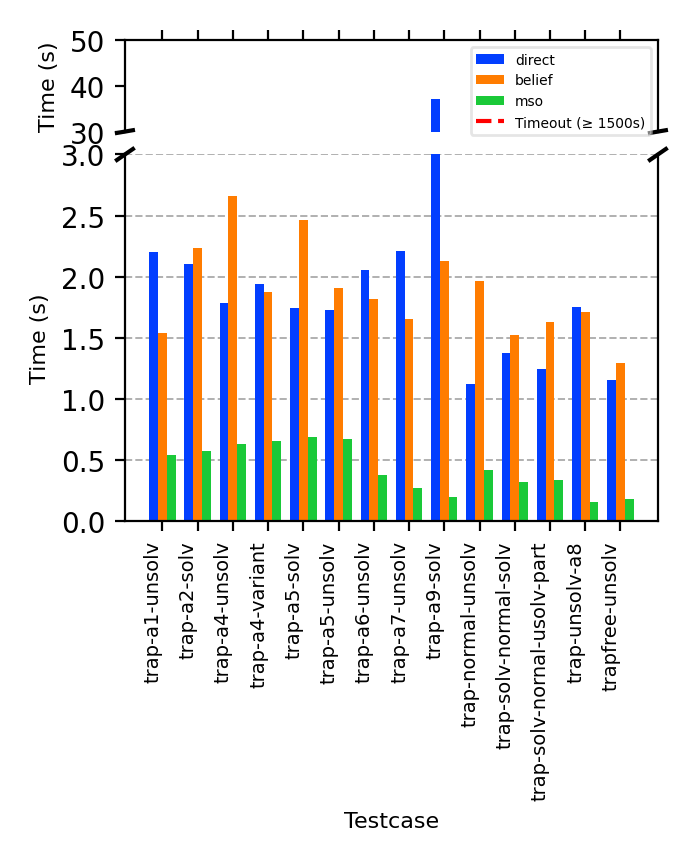}
    \caption{Runtime for the trap examples.}
    \label{fig:runtimeTrap}
\end{figure}

\section{Related Work}\label{sect:related}
Our work is related to previous investigations on \LTLf synthesis under partial observability, a problem that was originally investigated by \citet{de2016ltlf}. While related, our problem is novel: it requires the agent to simultaneously play two games in different arenas with distinct objectives, one corresponding to standard synthesis and the other to synthesis with partial information, thereby generalizing the problem. Further, \citet{tabajara2020ltlf} empirically investigate implementing algorithms for \LTLf synthesis under partial observability, translating the results from \citet{de2016ltlf} into practice.
They also discuss an MSO approach (though they do not discuss \QLTLf{} explicitly), the direct approach, and belief state construction. 
Our empirical results, however, are somewhat different; in our benchmarks, we never encountered situations where the direct approach or belief state construction solved instances that were not solvable using MSO, which they encounter in all their benchmarks.
This may be because we have both a main and backup specification, and the larger main specification often dominates the runtime; additionally, the experimental setups (i.e., amount of memory) were different.
Similarly related is work on \LTL synthesis under partial information~\citep{Kupferman2000, 2015ehlers_estimator}. However, we did not investigate embeddings into \LTL in our setting because of the generally better scalability of \LTLf synthesis~\citep{Zhu_De_Giacomo_Pu_Vardi_2020}.

There are both stochastic and non-stochastic related approaches. \citet{de2023a} consider computing best-effort \LTLf{}-strategies. While in both their setting and ours multiple variants of the environment are considered, their best-effort synthesis assumes a fully reliable and observable environment, which does not apply to our framework, and rather models the agent being uncertain about the specific environment and not errors in the input. In the stochastic setting, \citet{yu2024a} studied the trembling hand problem, which refers to scenarios where the agent may instruct a wrong action (with a certain probability). In contrast, in our setting, the unreliability is on the environment.
Multiple works use partially observable Markov decision processes to cope with uncertainty about the environment~\citep{hibbard2019unpredictable,lauri2022partially}.
More closely related is planning under partial observability; for example, \citet{aineto2023action} study planning where the agent's actions may fail up to k times, which is similar to our framework and could be modeled in it. 
Similarly, \citet{aminof2022verification} consider planning for multiple agents in a partially reliable environment simultaneously.
Our work is also related to work on planning with soft and hard \LTL{}/ \LDLf{} goals.
For example, \citet{rahmani2020you, rahmani2019optimal} consider the problem of satisfying an \LTL specification while guaranteeing that a subset of some soft constraints (expressed in \LDLf or \LTL, respectively) is satisfied (where they are ordered by priority). 
\citet{guo2018probabilistic} address the synthesis of control policies for MDPs that ensure high-probability satisfaction of \LTL tasks while optimizing total cost; their method too employs a product of automata, but additionally involves solving two linear programs.

\citet{2007WuExpressiveQltl} characterizes the expressive power of \QLTL{}, observing that one alternation already suffices to express all $\omega$-regular languages. Recently, there has been renewed interest in second-order quantification in infinite-trace variants of \LTL~\citep{piribauer2021quantified}; however, we are not aware of prior work on \QLTLf{}.

\section{Conclusion}
In this paper, we investigated reactive synthesis with backup goals for unreliable inputs using \LTLf{} as our specification language. 
We presented three algorithms, two of which match the 2EXPTIME-hardness result, of which the MSO approach performs best. Moreover, we showed how our problem can be seen as an instance of \QLTLf{} synthesis. While we investigated synthesis with unreliable input in the context of \LTLf, but it would be interesting to extend this study to other forms of specifications, possibly distinguishing the formalism used for the main goal and the backup one (i.e., using LTL safety specifications, \citet{AminofGSFRZ23}); furthermore one could here explore where the backup goal is satisfied earlier or later than the main goal. Additionally, it may be instructive to consider our problem in planning domains. Our techniques do not rely on probabilities and always ensure the backup condition is met, which is crucial for safety-critical scenarios. For less critical systems, this requirement may be relaxed by using quantitative techniques.
In this paper, we have also introduced techniques for synthesis in \QLTLf{}. \QLTLf{} deserves further study, including whether other problems can be cast naturally as \QLTLf{} synthesis.
\section*{Acknowledgments} 
We thank Antonio Di Stasio, Shufang Zhu and Pian Yu for insightful discussions on the topic of this paper. We also thank Lucas M. Tabajara for providing the source code for \LTLf synthesis under partial observability.  %
We would like to acknowledge the use of the University of Oxford Advanced Research Computing (ARC) facility in carrying out the experiments (http://dx.doi.org/10.5281/zenodo.22558). 
This research has been supported by the ERC Advanced Grant WhiteMech (No. 834228). 
\clearpage
\bibliographystyle{aaai}
\bibliography{aaai25}
\clearpage
\appendix
\newpage
\section{Proofs}\label{apx:proofs}

\subsection{Direct Approach}
\begin{theorem}\label{thm:syncProdAnd}
Let $A_1, A_2$ be DFAs over the same alphabet, i.e., $\mathcal G_i = (2^{\mathcal X} \times 2^{\mathcal Y}, S_i, s_{0, i}, \delta_i, F_i)$ (with $i \in \{ 1, 2 \}$, then a trace $t \in L(G_1 \otimes G_2)$ if and only if  $t \in L(G_1) \wedge t \in L(G_2)$. 
\end{theorem}
\begin{proof}    
From left-to-right assume that $t \in L(G_1 \oplus G_2)$. This is the case if and only if  running $t$ on the product automaton produces a sequence of states $((s_{0, 1}, s_{0,2}, \dots, (s_{i,1}, s_{j,2}))$ with $s_{i,1} \in F_1$ and $s_{i,2} \in F_2$. By construction of the transitions, this is the case if and only if running $t$ on $\mathcal G_1$ produces the state sequence $(s_{0,1}, \dots, s_{i,1})$ and running $t$ on $\mathcal G_2$, but this is the case iff $t_1 \in L(\mathcal G_1)$ and $t_2 \in L(\mathcal{G}_2)$.
\end{proof}

\begin{lemma}\label{lem:existentialAbstraction}
    For any trace $t$, we  have $t \in L(N_{\exists U_1, \dots, U_n})$ if and only if there is a trace $t' \sim_{- U_1, \dots, U_n } t$ such that $t' \in L(N)$.
\end{lemma}
 \begin{proof}
 	Assume that $t \in L(N_{\exists U_1, \dots, U_n})$.
    This implies that for running $t$ on $\mathcal N$ there is a possible sequence of states $s_0, \dots, s_k$ with $s_k \in F$.
    By definition, each transition $(s_i, t(i), s_{i+1})$ was possible since there was a way to modify $t(i)$ on the existentially abstracted variables such that the original automaton transitioned from $s_i$ to $s_{i+1}$. Putting these together gives a trace $t'$ that satisfies the definition. 
 	
 	The other direction is similar, using that we can arbitrarily modify the variables by existentially abstracting. 

  For the other direction, let $t$ be a trace s.t. there is a trace $t'$ with $t \sim_{- U_1, \dots, U_n} t'$ and $t' \in L(N)$. We need to show that $t \in L(N_{\exists U_1, \dots, U_n})$. Since $t' \in L(N)$, there is a sequence of states $(s_0, \dots, s_k)$ with $s_k \in F$ and $(s_i, t'(i), s_{i+1}) \in \delta$.
  Now, since $t \sim_{- U_1, \dots, U_n } t'$, we can use the same transition as we did for the modified trace, and thus to reach $s_k$ and therefore $t \in L(N_{\exists U_1, \dots, U_n})$.
 \end{proof}

\begin{lemma}\label{char:complementExistentialAbstraction}
    We have that $t \in \overline{\mathcal D(N_{\exists U_1, \dots, U_n})}$, if and only if for any trace $t' \sim_{-U_1, \dots, U_n} t$,we have $t' \notin L(N)$.
\end{lemma}
\begin{proof}
    Assume that $t \in L(\overline{\mathcal D(N_{\exists U_1, \dots, U_n})})$. This is the case if and only if $t \notin L((N_{\exists U_1, \dots, U_n}))$.  
    By \Cref{lem:existentialAbstraction} this is the case if and only if there is no trace $t' \sim_{-U_1, \dots, U_n } t$ s.t. $t' \in L(N)$, which is equivalent to claiming that for any trace $t' \sim_{-U_1, \dots, U_n}$ we have $t' \notin L(N)$.
\end{proof}

\begin{corollary}\label{col:characeriseExistential}
    We have that $t \in \overline{\mathcal D((\mathcal A_{\lnot \varphi_b})_{\exists U_1, \dots, U_n})}$, if and only if for any trace $t' \sim_{-U_1, \dots, U_n } t$,we have $t' \vDash \varphi_b$.
\end{corollary}
\begin{proof}
    Let $t \in L( \overline{\mathcal D((\mathcal A_{\lnot \varphi_b})_{\exists U_1, \dots, U_n})})$.
    Let $t' \in \mathsf{exp}_{\mathcal X_{unr}}(\mathsf{proj}_{X_{unr}}(t))$. By \Cref{char:complementExistentialAbstraction}, we then have that $t' \notin L(\mathcal A_{\lnot \varphi_b})$. But this implies that $t' \vDash \varphi_b$, just as we needed.

    For the other direction, assume that for any trace $t' \sim_{-U_1, \dots, U_n} t$,we have $t' \vDash \varphi_b$.
    Assume for the sake of deriving a contradiction that $t \notin \overline{\mathcal D((\mathcal A_{\lnot \varphi_b})_{\exists U_1, \dots, U_n})}$.
    This would imply  that  $t \in \mathcal D((\mathcal A_{\lnot \varphi_b})_{\exists U_1, \dots, U_n})$, which would imply (by \Cref{char:complementExistentialAbstraction}) that there is a trace $t'' \sim_{-U_1, \dots, U_n} t$ s.t. $t'' \vDash \lnot \varphi$.
    This would contradict our initial assumption.
\end{proof}
\rsThmCorrectProject*
\begin{proof}
    We have to argue that any strategy that guarantees reaching a final state in the DFA game, also ensures that synthesis under unreliable input is solved. 
    Thus let $\sigma$ be a strategy that wins in the synchronized product, thus there is an index $k$, where the transitions induced by the strategy of the DFA end in an accepting state.
    We have to show that $\sigma$ satisfies the conditions for synthesis under unreliable input. Therefore, let $\lambda = (X_0, \dots, ) \in {(2^{\mathcal X})}^\omega$ be arbitrary. Since $\sigma$ is a winning strategy in the DFA game, there is an index $k$ s.t., the state sequence induced by $t = ((Y_0, \sigma(Y_0)), \dots, )$ is in a final state. 
    We show that the same $k$ also satisfies the conditions for synthesis under unreliable input.

    First, since $\sigma$ wins in the DFA-game, we know that $t \in L(\mathcal A_{\varphi_m} \otimes \mathcal \overline{D((\mathcal A_{\lnot \varphi})_{\exists U_1, \dots, U_n})})$ we have (by \Cref{thm:syncProdAnd}) that $t \in L(\mathcal A_{\varphi_m})$, which is the case if and only if $t \vDash \varphi_m$.
     The other requirement to solve synthesis under unreliable input, follows using  \Cref{col:characeriseExistential,thm:syncProdAnd}.

    For the other direction, assume that $\sigma$ solves synthesis  under unreliable input. We have to show that $\sigma$ is also a winning strategy in the DFA game. 
    Therefore, let $\lambda = (X_0, \dots, ) \in {(2^{\mathcal X})}^\omega$ be arbitrary. Since $\sigma$ is strategy realizing synthesis under unreliable input, there is an index $k$ s.t., $t \vDash \varphi_m$ and for any ${t}' \sim_{-U_1, \dots, U_n} t$ we have that $t' \vDash \varphi_b$.  We need to argue that the DFA ends up in an accepting state.

    Again, it suffices to show both $t \in L(\mathcal A_{\varphi_m})$ and $t \in L(\mathcal D((\mathcal A_{\lnot \varphi})_{\exists U_1, \dots, U_n}))$.
    Again, the first is immediate from the correctness of the automaton. The second follows directly from \Cref{col:characeriseExistential}.
\end{proof}

\subsection{Belief State Approach}
\begin{lemma}\label{lem:equalOnUnobs}
For any two traces $t, t'$ with $t' \in \mathsf{exp}_{\mathcal X_{unr}}(\mathsf{proj}_{\mathcal X_{unr}}(t))$ we have that $t \in L(\mathcal G^{rel}_{\mathcal A})  \iff t' \in L(\mathcal G^{rel}_{\mathcal A})$
\end{lemma}
\begin{proof}
Follows by a simple induction, from the fact that the transitions are invariant under changing the unobservable variables.
\end{proof}

\begin{lemma}\label{lem:tracesBeliefState}
	If a trace $t$ induces  a play $(S_0, S_1, \dots, S_{t+1})$ in $G^{rel}_{A}$ and for any $s_{t+1} \in S_{t+1}$, then there is a trace $t' \sim_{-\U_1, \dots, U_n } t$ that induces a play $\rho' = (s_0, \dots, s_{t+1})$ on $\mathcal A$.  
\end{lemma}
\newcommand\mdoubleplus{\mathbin{+\mkern-10mu+}}

\begin{proof}
	The proof is by induction on the length of the trace. 
	The base case for traces of length 0 is trivial, as only $s_0 \in S_0$.
	For the inductive case, we consider the subtrace $(S_0, \dots, S_{i-1})$.
	
	 We know that since $s_{i} \in S_i$, there is a transition $(s, X \cup Y, s_{i})$. Moreover, using the inductive hypothesis, there is a trace $t'$ of length $i$ that leads to $s$ in the original automaton. If we consider $t = t'.(X, Y)$, this trace will lead to the state $s_i$.  
	 \end{proof}
\begin{lemma}\label{lem:recongitionIsTheSame}
	If $t \in L(\mathcal G_{A}^{Obs})$, then $t \in L(\mathcal A)$.
\end{lemma}
\begin{proof} 
	Use the fact that each step of the original automaton is also \emph{tracked} in the power sets, i.e. for the sequence of states $(s_0, \dots, s_t, s_{t+1})$ generated by running $t$ on $\mathcal A$, we will always have that $s_i \in S_i$ where $S_i$ are the states of the belief state automaton (\Cref{lem:tracesBeliefState}).
	
	But this implies that the last state of running $t$ on the original automaton is included in the last set $S_n$ we transition into, and since this is a final state of the belief state automaton, we know that $s_{t+1} \in F$ and thus $t \in L(A)$. 
\end{proof}

\begin{theorem}\label{lem:charBelStat}
For any trace $t$,  $t \in \mathcal L(G^{rel}_{\mathcal A_{\varphi_b}})$ if and only if  for any $t'$ with $t \sim_{-U_1, \dots, U_n} t'$ we have $t' \vDash \varphi_b$. 
\end{theorem}
\begin{proof}

Assume $t \in \mathcal L(G^{rel}_{\mathcal A_{\varphi_b}})$. Assume for the sake of deriving a contradiction that there is a trace $t' \in \mathsf{exp}_{\mathcal X_{unr}}(\mathsf{proj}_{\mathcal X_{unr}}(t')$, such that $t' \nVdash \varphi_b$.
But since $t \in L(G^{rel}_{\mathcal A_{\varphi_b}})$, by \Cref{lem:equalOnUnobs}, we would also have that $t' \in  L(G^{rel}_{\mathcal A_{\varphi_b}})$ and thus (by \Cref{lem:recongitionIsTheSame}) that $t' \in L(\mathcal A_{\varphi_b})$ and therefore $t' \vDash \varphi_b$ -- a contradiction.

For the other direction, assume for any $t'$ with $t' \sim_{-U_1, \dots, U_n} t$ we have $t' \vDash \varphi_b$ but assume for the sake of contradiction, that $t \notin L(G^{rel}_{\mathcal A_{\varphi_b}})$. This can only be the case if a non-accepting state is among the state-set the belief state automaton is in after running $t$.
	This would imply (by 
	 \Cref{lem:tracesBeliefState}) that there is a trace $t' \in \mathsf{exp}_{\mathcal X_{unr}}(\mathsf{proj}_{\mathcal X_{unr}}(t))$ s.t. $t' \nVdash \varphi_b$ (as it does not end in an accepting trace of the deterministic automaton). 
  But this contradicts our initial assumption.
\end{proof}

\thmCorrectBelief*
\begin{proof}
The proof is almost identical to the one for \Cref{thm:projectcorrect}.  
To see this, notice that \Cref{lem:charBelStat,col:characeriseExistential} together imply that the belief state and projection automaton both recognize the same traces, thus the correctness argument is identical.
\end{proof}

\subsection{Translation of Quantified \LTLf{} to MSO}

\begin{definition}[MSO Semantics]
    Given a domain $\mathcal U = \{ 1, \dots, n \}$, and two functions $v: \mathcal V_1 \rightarrow \mathcal U$ assigning first-order variables to members of $U$ and a $w: \mathcal V_2 \rightarrow \mathcal U$ assigning second-order variables to subsets of $U$, we define entailment for MSO: 
\begin{align*}
    u, w \vDash X(x) &\logeq u(x) \in w(X) \\ 
    u, w \vDash x < y &\logeq u(x) < u(w) \\
    u, w \vDash \varphi \wedge \psi &\logeq (u, w \vDash \varphi) \wedge (u, w \vDash \psi) \\ 
    u, w \vDash \lnot \varphi &\logeq \lnot (u, w \vDash \varphi) \\ 
    u, w \vDash \exists x. \varphi &\logeq \exists \hat{x} \in \mathcal U \text{ s.t. } u[x/\hat{x}], w \vDash \varphi \\ 
    u, w \vDash \exists X. \varphi &\logeq \exists \hat{X} \subseteq \mathcal{U} \text{ s.t. }  u, w[X/\hat{X}] \vDash \varphi       
\end{align*}

\end{definition}

\begin{definition}
    Given a trace $t$, we define the corresponding assignment of second order-variables by setting $w_t(A) = \{ i \mid A \in t(i) \}$ for any $A \in \mathcal P$.
    Then we define that $t, [i/x] \vDash \varphi$ is a shorthand for $[i/x], w_t \vDash \varphi$.
\end{definition}

 \begin{lemma}\label{lem:msotrans}
 	Let $\varphi$ be a \QLTLf formula and $t$ be a finite trace. 
Then $t, i \vDash \varphi$ iff $[x/i], w_t \vDash \varphi$.
 \end{lemma}
 \begin{proof}
 The proof is by induction on the formula with $i, t, x$ quantified.
 
 \begin{itemize}
 	\item Assume that $\varphi = A$. 
    Then $t, i \vDash A$ if and only if $A \in t(i)$, which is the case iff $A \in t(i)$, this is the case iff $i \in w_t(A)$ which again is the case iff $[x/i], w_t \vDash A(x)$. This concludes this step since $A(x) = \mathsf{mso}(A, x)$.
 		 
 	\item Assume that $\varphi = \lnot \psi$.
    Then we know that $t \vDash \lnot \psi$ if and only if it is not the case that $t, i \vDash \psi$, which by the IH, is equiavlent to it not being the case that $[x/i], w_t \vDash \mathsf{mso}(\psi, x)$.
    By semantics this is the case iff $[x/i], w_t \vDash \lnot \mathsf{mso}(\psi, x)$ and this is the case if and only if $[x/i], w_t \vDash \mathsf{mso}(\varphi, x)$.

 	\item Assume that $\varphi = \Next \psi$. 
 	We prove both directions separately.
 	 	For the forward direction, assume that $t, i \vDash \Next \psi$. This is the case if and only if $t, i+1 \vDash \psi$. By the inductive hypothesis, this is equivalent to $[y/(i+1)], w_t \vDash \mathsf{mso}_{\mathcal{B}}(\psi, y)$. Since $succ(x, y)$ is interpreted as the successor relation, we have $\sigma[x/i], w_t \vDash \exists y. succ(x, y) \wedge \mathsf{mso}(\psi, y)$. By the definition of the translation function, this is equivalent to $[x/i], w_t \vDash \mathsf{mso}(\Next \psi, x)$.
 	
 	For the backward direction, assume that $[x/i], w_t \vDash \mathsf{mso}(\Next \psi, x)$. By definition, this means $[x/i], w_t \vDash \exists y. succ(x, y) \wedge \mathsf{mso}(\psi, y)$. Thus there is a $\hat{y} \in \mathcal U$ such that 
 	$[x/i][y/\hat{y}], w_t \vDash succ(x, y) \wedge \mathsf{mso}(\psi, y)$.
 	This is the case if both $[x/i][y/\hat{y}], w_t \vDash succ(x, y) $ and $[x/i][y/\hat{y}], w_t \vDash \mathsf{mso}(\psi, y)$. The first implies that $(i, \hat{y})$ are successors, thus $\hat{y} = i+1$.
 	The second (noticing that $x$ does not occur free in $\mathsf{mso}(\psi, y)$) implies that $ [y/i+1], w_t \vDash \mathsf{mso}(\psi, y)$, which by the inductive hypothesis implies that $t, i+1 \vDash \psi$.
 	This directly implies that $t, i \vDash \Next \psi$.
 	\item Assume that $\varphi = \psi \lUntil \chi$.
 	Assume that $t, i \vDash \psi \lUntil \chi$. This is the case if there exists $j$, with $i \leq j \leq \last$, such that $t, j \vDash \chi$ and for all $k$ with $i \leq k < j$, $t, k \vDash \psi$. By the inductive hypothesis, this is equivalent to there existing $j \geq i$ such that $[x/j], w_t \vDash \mathsf{mso}(\chi, x)$ and for all $k$ with $i \leq k < j$, $[x/k], w_t \vDash \mathsf{mso}(\psi, x)$. By MSO semantics and the definition of our model (importantly, since $<$ and $S$ are interpreted as they are on $\mathbb N$), this is the case iff 
 	 $[x/i], w_t \vDash \exists y. (x \leq y \leq \text{last}) \wedge \mathsf{mso}(\chi, y) \wedge \forall z. (x \leq z < y \rightarrow \mathsf{mso}(\psi , z)).$
 	This is, by the definition of the translation function, equivalent to $\mathcal{I}_{t}, \sigma[x/i] \vDash \mathsf{mso}_{\mathcal{B}}(\psi \lUntil \chi, x)$.
  	\item Assume that $\varphi = \exists X. \, \psi$. 
    We know that $t, i \vDash \exists X. \varphi$ if and only if there is a $t'$ s.t. $t' \sim_{-X} t$ and $t', i \vDash \varphi$, by the IH this is the case iff $[x/i], w_{t'} \vDash \mathsf{mso}(\varphi, x)$; since the two variable assignments $w_t$ and $w_{t'}$  (as can be easily checked) at most differ in their $X$-assignments, this is the case iff $[x/i], w_{t} \vDash \exists X. \mathsf{mso}(\varphi, x)$, this is the case iff $[x/i], w_t \vDash \mathsf{mso}(\exists X. \varphi, x)$.
 \end{itemize}
 \end{proof}

As a direct corollary, we have a proof of the correctness statement from the main text:
\thmTranlsationmsocor* 
\begin{proof}
Directly follows from \Cref{lem:msotrans}.    
\end{proof}

\subsection{Quantified LTLf Synthesis}
\thmReductionToQltlf*
\begin{proof}
    From left to right, assume that $\sigma$ solves synthesis under unreliable input and let $t = ((X_0, Y_0), \dots,  (X_i, Y_i))$ be a trace with $Y_i = \sigma(X_0, \dots, X_{i-1})$. Then, trivially, we have that $t \vDash \varphi_m$. Thus, to establish that $t \vDash \varphi_m \wedge \forall U_1. \dots \forall U_n. \varphi_b$ we only need to show $t \vDash \forall U_1. \dots \forall U_n. \varphi_b$. 
    This can be shown considering that the semantics of $t \vDash \forall U_1. \dots. \forall U_n. \varphi_b$ unfolds into $\forall t' \in \mathsf{exp}_{\mathcal X_{unr}}(\mathsf{proj}_{\mathcal X_{unr}}(t)).\, t' \vDash \varphi_b$.

    Conversely, assume $\sigma$ synthesizes the \QLTLf formula. Thus, for any trace $t = ((X_0, Y_0), \dots,  (X_i, Y_i))$ with $Y_i = \sigma(X_0, \dots, X_{i-1})$, it holds that $t \vDash \varphi \wedge \forall U_1. \dots. U_n. \varphi_b$. Then trivially (since \QLTLf has the same semantics as \LTLf{} for its \LTLf{}-fragment), we have $t \vDash \varphi_m$. 
    Additionally, we also have that $t' \vDash \varphi_b$ for any $t'$ with $t' \in \mathsf{exp}_{\mathcal X_{unr}}(\mathsf{proj}_{\mathcal X_{unr}}(t))$, as essentially any trace in the expansion can be restored by using the universal quantification in \QLTLf.
\end{proof}

\begin{definition}
	Given a \QLTLf formula $\varphi$ in PNF with $k$-alternations, we can inductively define a deterministic automaton $A_{\varphi}$ that recognizes it: 
	
	\begin{itemize}
		\item If $\varphi$ has one alternation, we either have a bare \LTLf{}-formula, or a a chain of one or multiple quantifiers of the same type: 
		\begin{itemize}
			\item For bare formulas, we just use the \LTLf-automaton corresponding to $\varphi$.
			\item For a single existential alternation i.e. $\exists X_1, \dots, \exists X_n. \varphi$, we first construct NFA for $\varphi$ and then existentially abstract over the rest, i.e. ${(A_{\varphi})}_{\exists X_1, \dots, X_n}$, and determinize it.
			\item For a single universal alternation, i.e. $\forall X_1, \dots, X_n. \psi$, we know that this is the same as $\lnot \exists X_1 \dots \exists X_n. \lnot \psi$. We can construct a DFA  for $\psi$, then negate that, obtaining a DFA $A_{\lnot \psi}$, we then existentially abstract $(A_{\lnot \psi})_{\exists X_1, \dots, X_n}$ determinize and lastly negate once more, yielding $A_{\varphi} = \overline{D((A_{\lnot \psi})_{\exists X_1, \dots, X_n})}$. 
		\end{itemize}
		\item If the formula has a higher alternation count it is of the form $\varphi = (\exists / \forall)^+ \phi_k$ where $\phi_k$ has $k$ alternations. We thus, assume (for the inductive definition) to have an automaton $A_{\psi}$.
			\begin{itemize}
				\item If $\varphi = \exists X_1 \dots \exists X_n.  \phi_k$. We build the existentially-abstracted NFA $N_{\exists X_1, \dots, X_n}(A_{\phi_k})$, and determinise it; i.e. $A_{\psi} \coloneqq \mathcal D(N_{\exists X_1, \dots, X_n}(A_{\phi_k}))$.
				\item If $\varphi = \forall X_1, \dots X_n. \psi$, then we built the DFA $A_{\lnot \phi_k}$ (by negating the DFA provided by the inductive hypothesis). Then we build the NFA $N_{\exists X_1, \dots, X_n}$, and determinise and negate it; i.e. $A_{\varphi} = \mathcal \overline{D(\overline{N_{X_1, \dots, X_n}(A_{\lnot \psi})})}$.
			\end{itemize}
	\end{itemize}
	\end{definition}

\begin{theorem}\label{thm:qltlfCorrectness}
	The automaton $A_{\varphi}$ we compute correctly recognises $\varphi$ (for any \QLTLf formula $\varphi$ in PNF).
\end{theorem}
\begin{proof}
	We have to show that for any formula $\varphi$ and trace $t$, $t \vDash \varphi$ iff $t \in L(A_{\varphi})$. The proof is by induction on the alternation count. 
	\begin{itemize}
		\item Base Case:
		\begin{itemize}
			\item If $\varphi$ is a bare \LTLf{}-formula, then the correctness follows from the correctness of translating \LTLf into DFAs.
			\item Assume that $\varphi = \exists U_1. \dots \exists U_n. \psi$. Then, $t \in L(\exists X_1. \dots \exists X_n. \psi)$ iff there is a $t'$ with $t' \sim_{-X_1, \dots, X_n} t$ and $t' \vDash \psi$, this is by correctness of \LTLf translation the case if and only if $t' \in L(A_{\psi})$. But this is by \Cref{col:characeriseExistential} the case if and only if $t \in L({(A_{\psi})}_{\exists X_1, \dots, X_n})$; which is the case iff $t \in L(A_\varphi)$ (since determinization does not affect the language).
			\item Assume that $\varphi = \forall X_1. \dots \forall X_n. \psi$.
			Then $t \vDash \varphi$ if and only if for all traces $t'$ with $t \sim_{-X_1, \dots, X_n} t$ we have that $t' \vDash \psi$.
			This is equivalent to there being no trace $t'$ with $t' \sim_{-X_1, \dots, X_n} t$ with $t' \nvDash \psi$. 
			This is equivalent , by the IH, to there being no trace $t'$ with $t' \sim_{-X_1, \dots, X_n}$ having $t' \notin L(A_{\psi})$.
			
			This is equivalent to to there being no trace $t'$ with  $t' \sim_{-X_1, \dots, X_n} t$, we have that $t' \in L(A_{\lnot \psi})$. 
			But, this is equivalent it is not the case that there is a trace $t'$ with $t' \sim_{-X_1, \dots, X_n}$ and $t' \in L((A_{\lnot \psi})$.
			But this is the case iff it is not the case that $t \in L((A_{\lnot \psi})_{\exists X_1, \dots, X_m})$, this is the case iff $t \in L(\overline{D((A_{\lnot \psi})_{\exists X_1, \dots, X_n})})$.
		\end{itemize}
		\item Inductive case:
		\begin{itemize}
			\item Assume that $\varphi = \exists X_1 \dots \exists X_n \varphi_{k-1}$.
			By the inductive hypothesis, there is a DFA $A_{\varphi_{k-1}}$ that recognises $\varphi_{k-1}$. 
			Now $t \vDash \exists X_1 \dots \exists X_n. \varphi_{k-1}$ if and only if there is a trace $t'$ with $t' \sim_{-X_1, \dots, X_n} t$ and $t' \vDash  \varphi_{k-1}$. This is the case iff there is a trace $t'$ with $t' \sim_{-X_1, \dots, X_n} t$ and $t' \in L(A_{\varphi_{k-1}})$.
			But this is by \Cref{col:characeriseExistential} the case if and only if $t \in L((A_{\varphi_{k-1}})_{\exists X_1, \dots, X_n})$.
			This is the case if and only if $t \in L(A_{\varphi})$, since the determinisation does not affect the recognised language.
			\item Assume that $\varphi = \forall X_1 \dots \forall X_n \varphi_{k-1}$. 
			Then $t \vDash \varphi$ if and only if for all traces $t'$ with $t \sim_{-X_1, \dots, X_n} t$ we have that $t' \vDash \varphi_{k-1}$.
			This is equivalent to there being no trace $t'$ with $t' \sim_{-X_1, \dots, X_n} t$ with $t' \nvDash \varphi_{k-1}$. 
			This is equivalent, by the IH, to there being no trace $t'$ with $t' \sim_{-X_1, \dots, X_n}$ having $t' \notin L(A_{\varphi_{k-1}})$.
			
			This is equivalent to to there being no trace $t'$ with  $t' \sim_{-X_1, \dots, X_n} t$, we have that $t' \in L(A_{\lnot \varphi_{k-1}})$. 
			But, this is equivalent it is not the case that there is a trace $t'$ with $t' \sim_{-X_1, \dots, X_n}$ and $t' \in L((A_{\lnot \varphi_{k-1}})$.
			But this is the case iff it is not the case that $t \in L((A_{\lnot \psi})_{\exists X_1, \dots, X_m})$, this is the case iff $t \in L(\overline{D((A_{\lnot \varphi_{k-1}})_{\exists X_1, \dots, X_n})})$.

		\end{itemize}
	\end{itemize}
\end{proof}

\begin{theorem}\label{thm:qltlfRuntime}
		Computing the $A_{\varphi}$ for a \QLTLf{} formula $\varphi$ takes $(k+2)$-EXPTIME.
\end{theorem}
\begin{proof}
	The proof is by induction on the alternation count.
	\begin{itemize}
		\item If the alternation count is 0, we have three base cases:
		\begin{itemize}
			\item If the matrix is a \LTLf formula, then trivially it takes 2EXPTIME, as we need one exponential for creating an NFA and one for determinisation.
			\item If the formula is of the kind $\exists^* \varphi$, then we need one exponential for the NFA for $\varphi$, polynomial for the existential abstraction, and another exponential for the determinzation.
			\item If the formula is of the kind $\forall^* \varphi$, then we need one exponential to create the NFA for $\lnot \varphi$. Then, we need polynomial time for the existential abstraction, one exponential for determinizing and afterwards complementing is linear time. Thus, we overall have 2EXPTime.
		\end{itemize}
		\item For the inductive step, we have to consider two cases:
		\begin{itemize}
			\item If the formula is of the form $\exists^* \varphi_{k-1}$ then we need (k+1)-EXPTIME to create a DFA recognising $\varphi_{k-1}$. The existential abstraction NFA is polytime, however determinising it costs another exponential, yielding $(k+2)$-EXPTIME.
			\item If the formula is of the form $\forall^* \varphi_{k-1}$ then we need (k+1)-EXPTIME to create a DFA recognising $\varphi_{k-1}$. Negating a DFA is polynomial, then existential abstraction produces an NFA and is too polynomial. Determinization takes another exponential, negation is polynomial, thus we overall take $(k+2)$-EXPTIME.
		\end{itemize}
	\end{itemize}
\end{proof}

These allow us to show that we can do synthesis for arbitrary \QLTLf{} specifications.

\thmQLTLfkplustwoExpTime*
\begin{proof}
	Assume w.l.o.g. that $\varphi$ is in PNF; then we can compute an automaton that recognises $\varphi$ in (k+2)-EXPTIME, by \Cref{thm:qltlfCorrectness,thm:qltlfRuntime}. 
	We can then solve the game in polynomial time.
\end{proof}

\section{Sheep Example}\label{apx:sheep}
Consider a farmer who has $n$ sheep $S = \{ 1, \dots, n \}$ that are standing on the left side of a river. She only has a boat to cross the river. However, the boat can only move exactly two sheep at one point in time; additionally, some sheep $\mathcal D \subseteq S \times S$ do not like each other and cannot be moved across in the same time instance.
Additionally, the farmer believes some sheep currently like each other $L \subseteq \{ 1, \dots, n \} \times \{ 1, \dots , n \}$, but she may be mistaken about this. 
The main goal is to move all sheep across; however, if the farmer is mistaken about some of the sheep pairs belonging to set $D$ or $L$, then she is fine with moving a special subset $S' \subseteq S$ across.

Let us now describe how we translate this problem into an instance of \LTLf{} synthesis under unreliable input. We introduce $n$ atomic variables $\mathsf{left}_i$ that are true if the $i$-th sheep is currently on the left side. 
Additionally, we introduce output variables $\mathsf{move_i}$ that allow the agent to ask the environment to move a sheep. 
Lastly, for each pair $(i, j) \in \mathcal D \cup \mathcal L$, we introduce a variable $\mathsf{disallow_{i,j}}$ that indicates whether the sheep are allowed to move together.

\paragraph{Environment Description:} 
 We first describe the environment constraints: 
 \begin{itemize}
 	\item  Initially, all sheep are left i.e. $\bigwedge_{s \in \mathcal S} \mathsf{left}_s$ needs to be true.
 	\item If a sheep is not requested to move, its position stays the same: \[\bigwedge_{s \in \mathcal S} \Always\left (\Next \left ( \lnot \mathsf{move_i}  \right ) \wedge \lnot \mathsf{left_i} \supset \Next \left ( \lnot \mathsf{left_i} \right ) \right )\] and \[\bigwedge_{s \in \mathcal S} \Box\left (\circ  \left ( \lnot \mathsf{move_i}  \right ) \wedge  \mathsf{left_i} \supset \circ \left (  \mathsf{left_i} \right ) \right ).\]
 	
 	\item We also need to specify when the sheep actually move. We split this into generation for pairs $(i, j)$ such that sheep i and sheep j could be blocked and pairs that are always unblocked (as for efficiency, the variables about blocking only exist for potentially blocked pairs). 
 	
 	For pairs that will never be blocked, we introduce \[ \Box  \left (\left (\Next \left (\mathsf{move_i} \wedge \mathsf{move_j}  \right ) \wedge \mathsf{left_i} \wedge \mathsf{left_j}  \right )  \supset \Next   \left (\lnot \mathsf{left_i} \wedge \lnot \mathsf{left_j}  \right )  \right ) .\] 
 	
 	For possibly blocked pairs, we introduce the following two constraints:
 	\begin{align*} \Box   (\left (\Next \left (\mathsf{move_i} \wedge \mathsf{move_j} \wedge \lnot \mathsf{disallow_{i,j}}  \right )  \wedge \mathsf{left_i} \wedge \mathsf{left_j}   \right ) \\ \supset \Next   \left (\lnot \mathsf{left_i} \wedge \lnot \mathsf{left_j}  \right )  ),\end{align*}
    \item We do not need to specify other movement constraints, as when we do not specify whether an action must lead to the sheep being moved, the environment can adversarially choose not to move the sheep.
\begin{align*}
    & \Box \left( \Next \left( \mathsf{move_i} \wedge \mathsf{move_j} \wedge \mathsf{disallow_{i,j}} \right) \wedge \mathsf{left_i} \wedge \mathsf{left_j} \right) \\
    & \supset \Next \left( \mathsf{left_i} \wedge \mathsf{left_j} \right).
\end{align*}
 	\item Lastly, we need to specify the initial situation description, namely forcing the values of the disallow variables to never change, thus for the main specification (but not for the backup one), we force that for $(i,j) \in L$ we have $\Always(\lnot \mathsf{disallow}_{i,j})$ and for $(i, j) \in D$ we have $\Always(\mathsf{disallow}_{i,j})$.

 \end{itemize}
With $\varphi_{e}$ (and $\varphi_{e'}$ for the backup specification) , we denote the conjunction of the above formulas.

\paragraph{Agent Constraints:} Secondly, we have to specify the agent constraints $\varphi_{ag}$ as a conjunction of the following formulas: 
\begin{itemize}
	\item The agent has to ensure that exactly two move variables are active. Thus we set the agent formula to be \[ \Box(\mathsf{exactly-2-of}(\mathsf{move_1}, \dots, \mathsf{move_n})). \]
	To encode the exactly-2-constraint, we use a simple quadratic encoding. \end{itemize}
Lastly, we need to specify the goal; here it is that eventually, all of the sheep are moved, i.e. $\varphi_{goal} \coloneqq \Diamond \left ( \bigwedge_{i \in \mathcal S} \lnot \mathsf{left_i} \right )$; and for the backup specification we have $\varphi_{goal'} \coloneqq \Diamond \left ( \bigwedge_{i \in \mathcal S'} \lnot \mathsf{left_i} \right )$.

\paragraph{Partitioning:} As a partitoning, we set all the disallow variables as unobservables. 

In summary, we are trying to synthesize a strategy realising $\varphi_{ag} \wedge (\varphi_{e} \supset \varphi_{goal})$, with backup-goal $\varphi_{ag} \wedge (\varphi'_{e} \supset \varphi'_{goal})$, where the goal is to only have $\Diamond \bigwedge_{i \in \mathcal S'} \lnot \mathsf{left_i}$.

\section{Graph Example}\label{apx:trap}
Let us first describe the variables. We have $ \lceil \log_2 n \rceil$ variables $pos_i$, which represent the current position of the robot in the graph. Additionally, for each trap, we introduce a variable $t_i$ that denotes whether the trap is on or off. These are all of the input variables. 
The robot has a single output variable $\mathsf{left}$ -- indicating whether to go left or right.

\paragraph{Environment Description: }
Since the agent can only move left or right, we do not need any constraints. The constraints we generate for the environment are based on having conjuncts of the form $\Box (p(e_s) \wedge \Wnext(left) \supset \Wnext(p(e_t)))$ for any edge $e = (e_s, e_t)$, where $p(n)$ is the propositional formula corresponding to being in state $n$ (i.e. using the unique binary encoding of the state). 
For edges that are used in traps, we add $t_i$ or $\lnot t_i$ as a conjunct to the left-hand side and, of course, adapt the goal state analogously. If a vertex has only 0 or 1 edge, we specify that the agent remains at that position.

Additionally, we have a condition that the traps cannot change, as this would make it impossible . This is just by having formulas of the form $t_i \supset \Always(t_i)$ and $\lnot t_i \supset \Always(\lnot t_i)$.
\paragraph{Goal:} For the normal description, the goal is just to eventually reach one of the goal states; for the backup version, we allow the safety states as well. Thus, both are of the form $\varphi_{env} \supset \varphi_{goal}$.

\paragraph{Partitioning: }The partitioning has all trap variables $t_i$ as unobservable.

\section{Hiker Example}\label{apx:hiker}

We can model this as a problem of synthesis under unreliable input in the following way: 
The environment controls the following variables: 
\begin{itemize}
	\item \textbf{berry}: If true, signifies that at the current position on the hiking trail, there is a berry that the agent can eat.
	\item \textbf{poison}: If true and berry is true, it indicates that the berry at the current position on the trail is poisonous.
	\item \textbf{herbs}: If true, indicates that at the current position of the hiking path, there are medical herbs that the agent could take to possibly alleviate sickness now or later.
	\item \textbf{sick}: If true, indicates that the hiker is currently sick.
	\item \textbf{eot}: If true, signifies that the end of the hiking trail was reached.
	\item \textbf{inbag}: If true, the herbs are currently in the hiker's bag and could be used.
\end{itemize}

The agent has control of the following variables: 
\begin{itemize}
	\item \textbf{eat}: If true, it signifies that the agent wants to eat the berry (if existent) at the current position of the hiking path.
	\item \textbf{takeMedication}: Take the herbal medication; if the hiker was sick before, they will no longer be.
	\item \textbf{collectMedication}: If there is herbal medication, the hiker can collect it (notice that their bag only allows them to store one piece of herbs).
\end{itemize}

Our formula, as in the other examples, will be of the shape $\varphi_{env} \supset \varphi_{ag}$.

\paragraph{Environment Description:}
We can describe the parts of the environment formula similar to SSA axioms in Reiter's Basic Action Theory. Thus, we basically have one successor state axiom for each of the environment variables. 

We do not have this for berry, poison and herbs, as those are randomized by the environment; however, we here include the constraint that poison is only true when berry also is true and that berries and herbs cannot be there at the same time.
\[ \Always\left ( berry \supset \lnot herbs \right ) \]
\[ \Always\left ( poison \supset berry \right ) \]

For sickness, we can notice that the hiker is sick when either they were sick before and have not both had medication available and taken medication or when they ate a poisonous berry.
This leads to the following successor state axioms:
\begin{align*}
	\Next sick &\equiv \Next \top \wedge ( \Next eat \wedge  berry \wedge  poison  ) \\ &\vee (sick \wedge \lnot (inbag \wedge takeMedication))
\end{align*}

For the end of the trail, we only have to specify that it is reached after a certain number of steps (here $k$). 
\[ \Wnext^{k}(eot) \wedge \bigwedge_{1 < k' < k} \Wnext^{k'}(\lnot eot) \]
 
 Additionally, we describe that once eot is true, it will stay true and no more berries appear before the hiker.
 \[ \Always(eot \supset \Wnext(eot)) \wedge \Always(eot \supset \lnot \textsf{berry}) \] 
  
 To control realizability, we can influence whether there is medicine along the trail. For the realizable variant, we  force that medication is available somewhere along the trace before the end:
 \[ \Wnext^{k - 3}(medication) \]

 For inbag, we have that it is true if either we have collected herbs in the last step (and there were herbs) or if we already had herbs and have not used them.
 \begin{align*}
 	\Next inbag \equiv \Next \top \wedge (herbs \wedge \Next collectMedication) \\ \wedge (inbag \wedge \lnot (takeMedication))
 \end{align*} 

\paragraph{Goal.} The agent's main goal consists of three parts:  The hiker wants to eat all non-poisonous berries that they cross during their hiking. Additionally, they want to reach the end of the trail. 
\begin{align*}
	\varphi_{ag} \coloneqq \Eventually (eot) \wedge \Always(berry \wedge \lnot poison \supset \Wnext(eat))
\end{align*} 
Lastly, for the initial state, we specify that the hiker is not sick and, of course, that it is not the end of the trace. 
\[ \lnot sick \wedge \lnot eot \]
\paragraph{Partitioning:}
For the partial observability variant, we set as unobservables $\mathcal X_{unr} = \{ poison \}$.

\section{Implementation Details}\label{apx:reprod}

The source code of the modified version of Syft that we use for our experiments is included in the supplementary material. 

Each program can be run using the command $\texttt{./Syft input.ltlf paritioning.part 0 <mode>}$. 
With \texttt{mode}, it is possible to select the approach by setting it to either \texttt{direct, belief, mso}, for the direct approach, the belief-state construction or the encoding into MSO, respectively.

\textbf{Input file format:}
It is important to notice that since our problem is different from standard \LTLf{}-synthesis, the layout of the input files is slightly changed.
Here, \texttt{input.ltlf} is a file containing two lines defining the main and backup \LTLf{}-formula, respectively. 

The formula itself is specified in Syft's LTLf syntax using $\texttt{N}$ for next, $\texttt{X}$ for Weak Next, $\texttt{G}$ for always, $\texttt{F}$ for eventually and usual symbols for propositional connectives.

\textbf{Partitioning file:}
The partition file contains the necessary information. It has the following format: 
\begin{lstlisting}
    .inputs: a b c 
    .outputs: x y z
    .unobservables: b c
\end{lstlisting}
This would correspond to a partitioning with $\mathcal X = \{ a, b, c \}, \mathcal Y = \{ x, y, z \}$ where $\mathcal X_{rel} = \{ a \}$ and $\mathcal X_{unr} = \{ b, c \}$.
Notice that we \textbf{require} each unreliable input variable to be listed \textbf{twice}: Once in the inputs and in the list of unobservables. 

\paragraph{Tests.}
The tests can be found in the \texttt{tests/} directory. Each test is in its own subdirectory and consists of a partitioning file \texttt{test.part}, a corresponding formula file \texttt{test.ltlf} and a file \texttt{expected} that either contains a 0 or 1, indicating whether the test is unrealizable or realizable, respectively.

The tests can be run using the python script \texttt{runTests.py} in the directory \texttt{scripts/runTests}. As an argument it takes the implementation type (i.e., mso, direct or belief), the other arguments can be viewed by calling it with \texttt{-h} to display the help.

\paragraph{Test Generation.}
The tests (for sheep, hiker and graph) can be generated using the Python scripts in the \texttt{generators/} folder.
The other test cases (i.e. \texttt{simpl-unrel}) are not auto-generated.
\subsection{Compilation.} 
Syft requires some libraries to be installed, this includes CUDD (vesion 3.0.0), Boost (v. 1.82) and MONA (used version MONA v1.4-18). The code was compiled using g++ (version 13.2.0). We have included the version of CUDD we used in the supplementary material (\texttt{extern/} folder), as different versions may lead to incompatibilities. 

There are detailed installation instructions for Ubuntu-based systems in the  \texttt{Syft/INSTALL} path of the archive.

\section{Supplementary figures}\label{apx:figures}
\subsection{Overhead}
To evaluate the overhead introduced by synthesis under unreliable inputs, we compared the runtime of our approach with the runtime required for separately synthesizing strategies for the main formula under full observability and the backup formula under partial observability.\footnote{This can be triggered using the command line option \texttt{-disregard=main/backup} in the testing script.}

In \Cref{fig:boxplot}, we present the ratio of our algorithm's performance relative to synthesizing only the main formula, only the backup formula under partial observability, or the sum of these two. A ratio of 1 indicates identical performance, while a ratio of 2 means our approach takes twice as long, and so forth. As shown in the figure, when compared with the sum of individual runtimes, our approach is at most 2.3 times slower, indicating a linear overhead. Furthermore, the mean and median ratios are both less than 1, with a standard deviation of 0.3.

We also provide a plot of the actual runtimes in \Cref{fig:overhead}, using the same color scheme for direct, belief-state, and MSO techniques as in the other figures.

\begin{figure*} 
    \centering
    \includegraphics[scale=0.5]{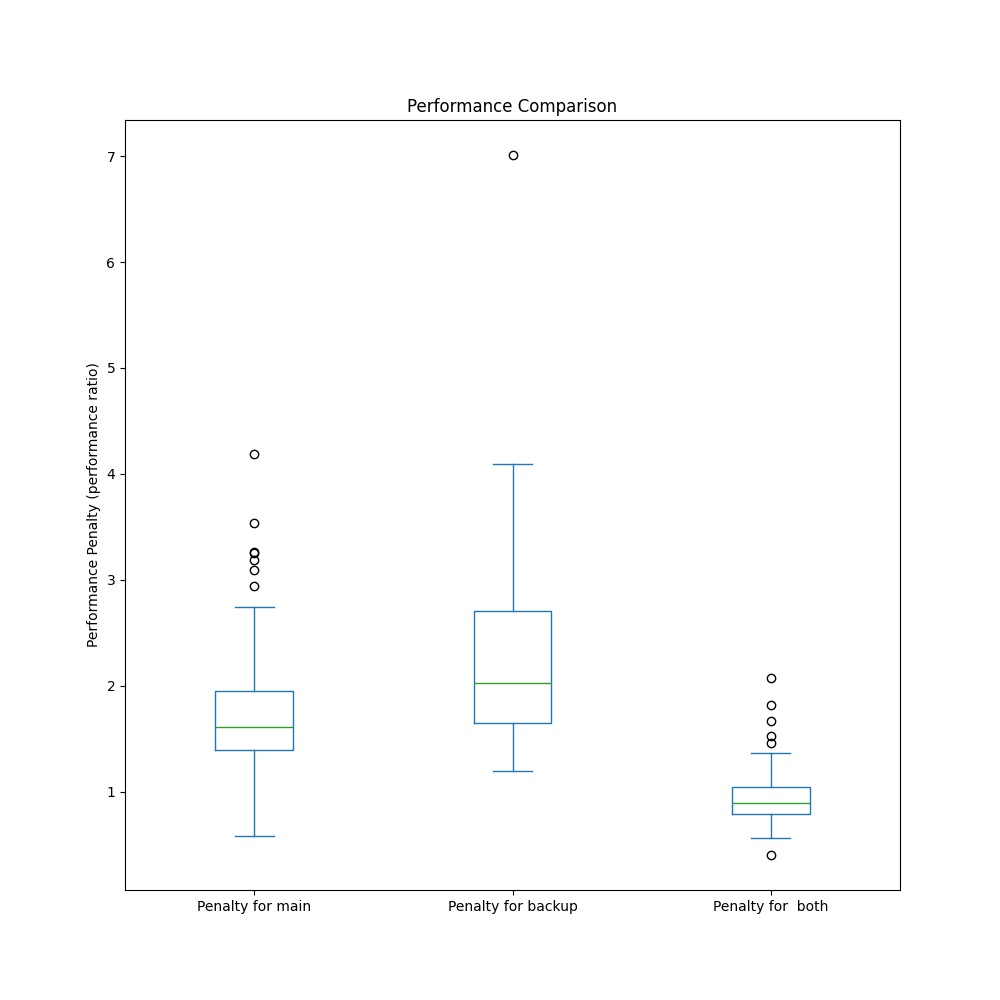}
    \caption{Boxplot of ratio of MSO solution for synthesis under unreliable input}
    \label{fig:boxplot}
\end{figure*}
\begin{figure*}
    \centering
    \includegraphics[scale=0.55]{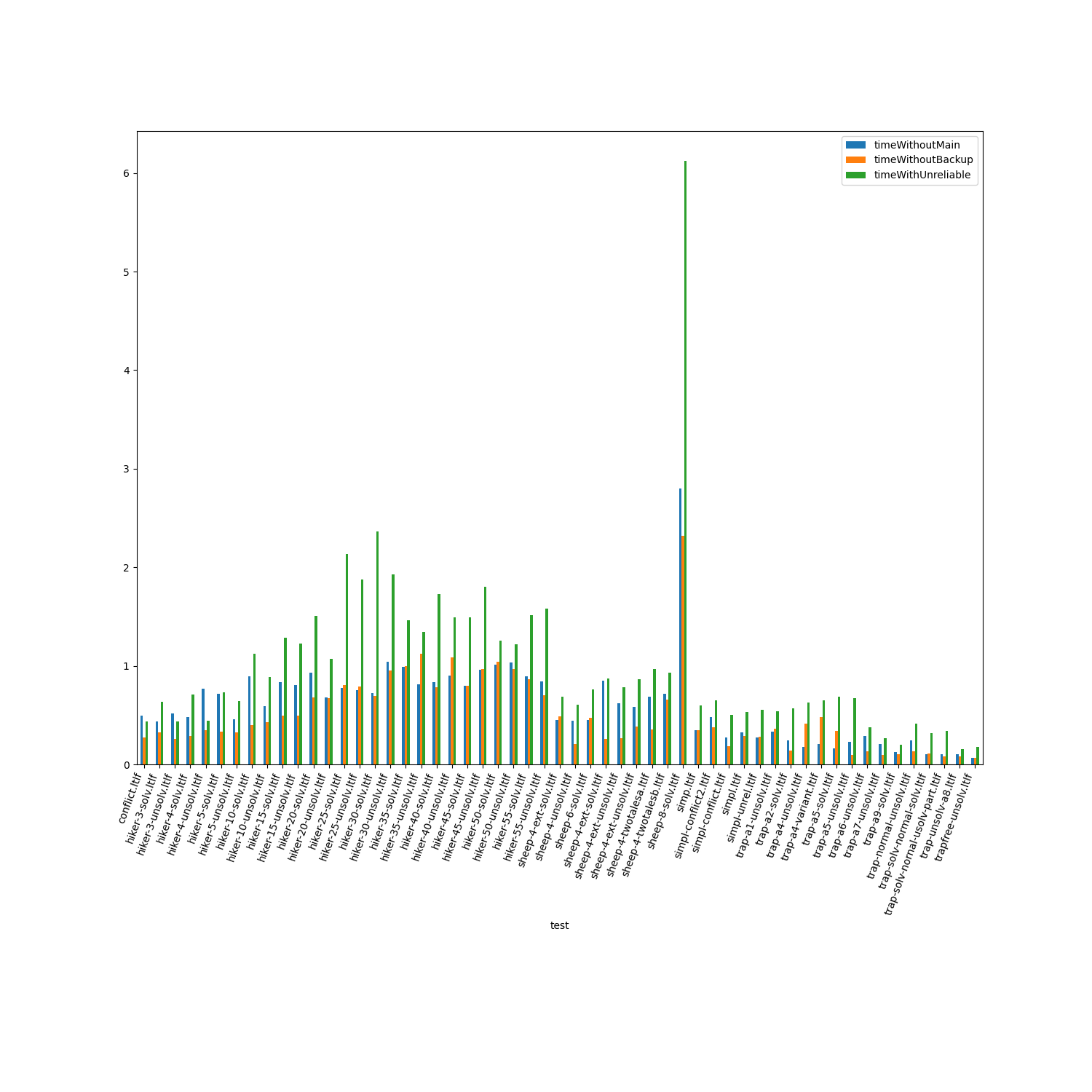}
    \caption{Overhead on tests}
    \label{fig:overhead}
\end{figure*}

\end{document}